\documentclass[10pt,letterpaper]{article}
\usepackage[T1]{fontenc}
\usepackage{times}
\usepackage[textwidth=5.75in,textheight=9in,centering]{geometry}
\usepackage{microtype}
\usepackage{titlesec}
\usepackage{titling}
\usepackage{enumitem}
\usepackage[authoryear,round]{natbib}
\setcitestyle{citesep={;},aysep={,},yysep={;}}

\titleformat{\section}{\large\scshape\raggedright}{\thesection}{1em}{}
\titleformat{\subsection}{\normalsize\scshape\raggedright}{\thesubsection}{1em}{}
\titleformat{\subsubsection}{\normalsize\scshape\raggedright}{\thesubsubsection}{1em}{}
\titlespacing*{\section}{0pt}{2ex plus .5ex minus .2ex}{1.5ex plus .3ex minus .2ex}
\titlespacing*{\subsection}{0pt}{1.8ex plus .5ex minus .2ex}{.8ex plus .2ex}
\titlespacing*{\subsubsection}{0pt}{1.5ex plus .5ex minus .2ex}{.5ex plus .2ex}
\setlist{topsep=4pt,itemsep=2pt,parsep=0pt,partopsep=0pt}
\pretitle{\noindent\LARGE\scshape\raggedright}
\posttitle{\par\vspace{1em}}
\preauthor{\noindent\normalsize\begin{tabular}[t]{@{}l@{}}}
\postauthor{\end{tabular}\par}
\renewcommand{\and}{\end{tabular}\hspace{10em}\begin{tabular}[t]{@{}l@{}}}
\predate{}
\postdate{}
\renewenvironment{abstract}
  {\begin{center}\large\scshape\abstractname\end{center}\begin{quote}\normalsize}
  {\end{quote}\vspace{1ex}}
\usepackage{amsmath,amsfonts,bm}

\def\eqref#1{equation~\ref{#1}}

\def\1{\bm{1}}

\DeclareMathAlphabet{\mathsfit}{\encodingdefault}{\sfdefault}{m}{sl}
\SetMathAlphabet{\mathsfit}{bold}{\encodingdefault}{\sfdefault}{bx}{n}

\newcommand{\E}{\mathbb{E}}

\newcommand{\R}{\mathbb{R}}

\DeclareMathOperator{\vecop}{vec}
\DeclareMathOperator{\tr}{tr}

\usepackage{url}
\usepackage{amsmath}
\usepackage{amsthm}
\usepackage{algorithm}
\usepackage{algpseudocode}
\usepackage{graphicx}
\usepackage{amssymb}
\usepackage{aliascnt}
\usepackage[hidelinks]{hyperref}
\hypersetup{
  pdftitle={FastKron: Efficient Quantization with Kronecker-Factored Hessians},
  pdfauthor={Johann Birnick and Rayan Saab}
}
\usepackage[capitalize,noabbrev,nameinlink]{cleveref}
\makeatletter
\def\theHALG@line{\thealgorithm.\arabic{ALG@line}}
\makeatother
\usepackage{mathtools}
\usepackage{booktabs}
\usepackage{xcolor}
\usepackage{pgfplots}
\pgfplotsset{compat=1.18}

\definecolor{colorfastkron}{HTML}{2A78D6}
\definecolor{colorantidiagonal}{HTML}{EB6834}
\definecolor{colorgptqrecursive}{HTML}{4A3AA7}

\theoremstyle{plain}
\newtheorem{theorem}{Theorem}[section]
\newaliascnt{proposition}{theorem}

\aliascntresetthe{proposition}
\crefname{proposition}{Proposition}{Propositions}
\newaliascnt{lemma}{theorem}

\aliascntresetthe{lemma}
\crefname{lemma}{Lemma}{Lemmas}
\newaliascnt{corollary}{theorem}
\newtheorem{corollary}[corollary]{Corollary}
\aliascntresetthe{corollary}
\crefname{corollary}{Corollary}{Corollaries}

\title{FastKron: Efficient Quantization with Kronecker-Factored Hessians}

\author{
\textbf{Johann Birnick} \\
Department of Mathematics \\
University of California San Diego \\
La Jolla, CA 92093, USA \\
\texttt{jbirnick@ucsd.edu}
\and
\textbf{Rayan Saab} \\
Department of Mathematics and HDSI \\
University of California San Diego \\
La Jolla, CA 92093, USA \\
\texttt{rsaab@ucsd.edu}
}

\newcommand{\Z}{\mathbb{Z}}
\newcommand{\cL}{\mathcal{L}}
\newcommand{\cR}{\mathcal{R}}
\newcommand{\Oh}{\mathcal{O}}

\date{}

\begin{document}

\maketitle

\begin{abstract}
We accelerate a family of algorithms for neural network quantization which utilize a two-sided version of the GPTQ/LDLQ algorithm.
Standard GPTQ-style adaptive rounding uses one-sided correlation information derived from input activations.
A natural two-sided extension can additionally capture correlations across output channels.
It utilizes a general Kronecker-factored approximation of the weight matrix curvature.
This approach has been used in BoA and YAQA.
But making a concrete algorithmic implementation of this two-sided GPTQ variant is nontrivial.
BoA uses a large number of sequential steps, while YAQA improves over the sequential depth but still has quartic total cost.
We introduce FastKron, an efficient algorithmic implementation that combines anti-diagonal parallelism with a recursive divide-and-conquer construction.
For an $m\times n$ weight matrix, FastKron uses $\Oh(m+n)$ sequential steps while reducing the total work from $\Oh(m^2n^2)$ to $\Oh(mn(m+n))$.
Thus, it matches the cubic scaling of GPTQ while exploiting richer curvature information.
Moreover, FastKron is modular with respect to both the base quantizer and the Hessian estimator.
We also provide practical benchmarks, consider a range of Hessian approximations that FastKron can be used with, and provide an efficient technique to compute these Hessians.
\end{abstract}

\begin{center}
\fcolorbox{blue}{blue!8}{\parbox{\dimexpr\linewidth-2\fboxsep-2\fboxrule\relax}{\centering\color{blue!70!black}
\textbf{NOTE:} The previous version of this paper used the name ``BaKron''.
We have decided to change the name to ``FastKron''.
When citing the paper, please refer to the algorithm as ``FastKron''.}}
\end{center}

\section{Introduction}
\label{sec:introduction}

Consider the problem of quantizing the weights of a neural network post-training.
We are given a trained neural network, and for each linear unit with weight matrix $W \in \R^{m \times n}$, we wish to replace $W$ by another weight matrix $V$ of low numerical precision, while maintaining overall accuracy of the network.
Each linear unit $W$ is quantized separately.
However, simply rounding each entry of $W$ to the closest element of the quantization alphabet is suboptimal.
Instead, one can take the curvature of the activation or model space into account.

The well-known GPTQ/OPTQ algorithm by \citet{frantar2023optq} uses \emph{input correlations} to inform the quantization algorithm.
Concretely, the algorithm aims to find a $V$ that approximately minimizes $\E \lVert W a - V a \rVert_2^2$, where $a \in \R^n$ is an input to the linear unit, and in practice the expectation is taken over a calibration dataset.
Hence the Gram matrix $\E[a a^T]$ is used to inform the geometry of the GPTQ algorithm.
More concretely, when $V$ is vectorized into $v = \vecop(V) \in \R^{m n}$, the Hessian of the above quadratic optimization problem (up to a factor of 2) is exactly $\E[a a^T] \otimes I_m$.

The identity matrix $I_m$ that appears as the second factor is the manifestation of the fact that GPTQ quantizes all \emph{output features} independently and identically.
In this paper we consider a natural extension of the GPTQ algorithm to the more general setting of an arbitrary Kronecker-factored Hessian of the form $A \otimes B$, where $A \in \R^{n \times n}$ and $B \in \R^{m \times m}$ are arbitrary positive definite matrices.
In particular, $B$ is not required to be the identity matrix, but if it is, then the algorithm recovers standard GPTQ.

This natural extension has been used in a functionally equivalent form in previous papers by \citet{kim2025boa} (BoA) and \citet{tseng2026modelpreserving} (YAQA); however, the concrete algorithmic implementation in these previous works is computationally suboptimal.
\citet{kim2025boa} run a straightforward implementation which requires $\Oh(mn)$ sequential steps and $\Oh(m^2 n^2)$ total work.
\citet{tseng2026modelpreserving} improve on that with regard to parallelization, requiring only $\Oh(m+n)$ sequential steps, but still using $\Oh(m^2 n^2)$ total work.

\textbf{Contributions.}
This paper concerns the efficient implementation of this quantization method.
\begin{enumerate}
\item We describe and accelerate the two-sided GPTQ algorithm to bring the total work down to just $\Oh(m n (m+n))$, while still requiring only $\Oh(m+n)$ sequential steps.
  This is achieved with a recursive divide-and-conquer approach as an additional algorithmic ingredient.
  The resulting algorithm, \emph{FastKron}, matches the complexity of standard GPTQ, which also needs cubic $\Oh(m n^2)$ total cost and linear $\Oh(n)$ sequential steps, while allowing the quantization algorithm to capture richer geometric information; see \cref{tab:complexity}.
  We also compare benchmarked empirical runtimes and list different Hessian approximation methods.

\item We propose an improved quantization flow for computing the Kronecker-factored Hessian approximation with backpropagation.
  This improves over the way YAQA computes the Hessians, significantly decreasing memory requirements; see \cref{tab:backprop}.

\end{enumerate}

\textit{Concurrent work.}
\citet{chen2026gptq2d} also consider the acceleration of two-sided GPTQ and obtain an algorithm with the same cubic complexity as FastKron.
Our work is fully independent, and the cubic-time algorithm presented here predates that preprint.\footnote{A version of this work has been submitted previously for publication. The cubic-time algorithm presented here was documented in the written record of the review process before the above preprint appeared on arXiv.}
We also note that they use a different algorithmic technique to get the complexity down to $\Oh(mn(m+n))$.

\section{The FastKron Algorithm}
\label{sec:algorithm}

\subsection{Notation and Conventions}
\label{sec:notation}

Recall that the goal is to approximate $W \in \R^{m \times n}$ with a low-precision matrix $V$.
We model the quantized matrix $V$ as an integer matrix, $V \in \Z^{m \times n}$.
In practice, each row and/or column would be additionally scaled by a factor $s \in \R_{>0}$, the integers might be bounded (clipped), or entirely different alphabets, e.g., special lattices, might be used for $V$.
This, however, does not change the algorithms.
For all algorithms described in this paper, one may simply replace ``$\mathrm{round}$'' by a ``$\mathrm{quantize}$'' function that maps a scalar to the closest element of the quantization alphabet.
For example, \citet{tseng2026modelpreserving} (YAQA) use \cref{alg:fastkronantidiagonal} with a nontrivial base quantizer.

If $A$ is a positive definite matrix, then $\mathrm{LDL}(A)$ denotes the unique lower-triangular matrix $L$ with unit diagonal entries $L_{i,i} = 1$ that satisfies $A = L D L^T$ with some diagonal matrix $D$.
In other words, it is the Cholesky factor of $A$ with columns normalized by their diagonal entries.
The positive definite matrices used in this paper are to be interpreted as Gram matrices of some inner product, and are often referred to as \emph{curvature} matrices or Hessians; see \cref{app:hessians} for more details.

Recall that the Kronecker product satisfies $(A \otimes B) (A' \otimes B') = (A A') \otimes (B B')$.
In particular, matrix factorizations like the Cholesky/QR/LDL decompositions distribute over the Kronecker product.
For an $m \times n$ matrix $W$ we define $\vecop(W)$ to be the column-major flattening of $W$ into a vector of length $m \cdot n$.
This defines an isomorphism of vector spaces $\vecop\colon \R^{m \times n} \to \R^{m n}$ which satisfies:
\begin{equation}
\label{eq:vec}
(A \otimes B) \cdot \vecop(W) = \vecop(B W A^T)
\qquad\qquad
\lVert \vecop(W) \rVert_2 = \lVert W \rVert_F
\end{equation}

Submatrices of an $m \times n$ matrix $W$ are denoted as $W_{i:i', j:j'}$.
Indices start at \emph{one}, and ranges \emph{exclude} the start and \emph{include} the end, e.g., $W_{0:m,0:n}$ denotes the full matrix.
When $A$ is a square matrix, define $\mathrm{diag}(A)$ to be the diagonal \emph{square matrix of the same size} (not a vector), which has the same diagonal as $A$ and zeros everywhere else.
When $W$ is an $m \times n$ matrix, we define its $k$\textsuperscript{th} anti-diagonal as the collection of all entries $W_{i,j}$ with $i+j = k$.
Thus, $W$ has $m+n-1$ anti-diagonals, with indices $k=2,\ldots,m+n$.
We further define $\mathrm{antidiag}_k(W)$ to be the \emph{$m \times n$ matrix} which keeps the $k$\textsuperscript{th} anti-diagonal of $W$ and has all other entries set to $0$.
Similarly, $\mathrm{antidiag}_{k:k'}(W)$ denotes a matrix that keeps the anti-diagonals $k+1,\ldots,k'$ and is zero elsewhere.

\subsection{FastKron-naive: GPTQ in the Vectorized Weight Domain}
\label{sec:gptq_flattened}

\begin{algorithm}[h]
  \caption{GPTQ}
  \label{alg:gptq}
  \begin{algorithmic}
    \State {\bfseries Input:} $W \in \R^{m \times n}$, $L = \mathrm{LDL}(A^{-1}) \in \R^{n \times n}$
    \For{$j=1$ {\bfseries to} $n$}
    \State $v \gets \mathrm{round}(W_{:,j})$
    \State $\Delta \gets v - W_{:,j}$
    \State $W \gets W + \Delta \cdot (L_{:,j})^T$ \Comment{update $W$ (propagate error)}
    \EndFor
    \State {\bfseries Output:} $W$ (which satisfies $W \in \Z^{m \times n}$)
  \end{algorithmic}
\end{algorithm}

Recall the GPTQ algorithm, described in \cref{alg:gptq}.
It aims to minimize $D_A = \lVert (V - W) A^{1/2} \rVert_F^2$ for some positive definite $A \in \R^{n \times n}$.
The matrix $A$ informs the algorithm of the correlation of module inputs $a \in \R^n$.
In practice, $A = \E[a a^T]$, where the expectation is over a calibration dataset.
The Hessian of the quadratic loss $\frac{1}{2} D_A$ with respect to $\vecop(V)$ is $A \otimes I_m$.
The $I_m$ factor is the manifestation of the fact that \cref{alg:gptq} treats all output channels (all rows of $W$) independently.

Naturally, one can additionally consider a positive definite output curvature matrix $B \in \R^{m \times m}$ and the resulting loss $D_{A,B} = \lVert B^{1/2} (V - W) A^{1/2} \rVert_F^2$, whose Hessian with respect to $\vecop(V)$ is $A \otimes B$.

As already observed by \citet{kim2025boa}, there is a natural extension of GPTQ that targets $D_{A,B}$,
and it can be described as simply applying \cref{alg:gptq} (standard GPTQ) to the vectorized weight matrix $\hat{W} := \vecop(W)^T \in \R^{1 \times mn}$ and ``input'' curvature $\hat{A} := A \otimes B \in \R^{mn \times mn}$.
Via \cref{eq:vec} we can describe this in unvectorized weight space, giving \cref{alg:fastkronnaive}.
This is the basic extension of GPTQ to a two-sided curvature.
The goal of the following subsections is to accelerate \cref{alg:fastkronnaive}.

\begin{algorithm}[h]
  \caption{FastKron-naive (GPTQ in vectorized weight domain)}
  \label{alg:fastkronnaive}
  \begin{algorithmic}
    \State {\bfseries Input:} $W \in \R^{m \times n}$, $L^{(A)} = \mathrm{LDL}(A^{-1}) \in \R^{n \times n}$, $L^{(B)} = \mathrm{LDL}(B^{-1}) \in \R^{m \times m}$
    \For{$j=1$ {\bfseries to} $n$}
    \For{$i=1$ {\bfseries to} $m$}
    \State $v \gets \mathrm{round}(W_{i,j})$
    \State $\Delta \gets v - W_{i,j}$
    \State $W \gets W + L^{(B)}_{:,i} \cdot \Delta \cdot (L^{(A)}_{:,j})^T$ \Comment{update $W$ (propagate error)}
    \EndFor
    \EndFor
    \State {\bfseries Output:} $W$ (which satisfies $W \in \Z^{m \times n}$)
  \end{algorithmic}
\end{algorithm}

\subsection{FastKron-antidiagonal: Batch Processing for \texorpdfstring{$m+n$}{m+n} Sequential Steps}
\label{sec:fastkron_antidiagonal}

A big issue with \cref{alg:fastkronnaive} is that it requires $m \cdot n$ sequential steps.
Therefore it can only be employed for small matrices.
For example, \citet{kim2025boa} use the algorithm in this form, and they overcome the computational burden by only quantizing matrices from the attention block and taking a separate curvature measure for each attention head, so that they can process multiple attention heads in parallel while each attention head only yields small matrices.

To bring the sequential depth down, one has to process some of the $m \cdot n$ steps in parallel.
This can be done by analyzing the dependency graph among the $W_{i,j}$, as first shown by \citet{tseng2026modelpreserving}.

The main observation is that the update caused by an entry $W_{i,j}$ only updates entries to the bottom-right of itself, i.e., it only updates entries $W_{i',j'}$ which satisfy $i'\geq i$ and $j' \geq j$.
Conversely, this means that $W_{i',j'}$ will only ever be changed by updates caused by entries to the top-left of itself.
Thus one can round $W_{i,j}$ and perform the corresponding update to $W$ as soon as all entries above the anti-diagonal $i+j$ have been processed.
As shown in \cref{thm:equiv}, this allows one to go from anti-diagonal to anti-diagonal, and process each one in parallel, via:

\begin{algorithmic}
  \For{$k=2$ {\bfseries to} $m+n$} \Comment{iterate sequentially over anti-diagonals}
  \For{$(i,j)$ in anti-diagonal $k$ {\bfseries in parallel}} \Comment{process each anti-diagonal in parallel}
  \State $v \gets \mathrm{round}(W_{i,j})$
  \State $\Delta \gets v - W_{i,j}$
  \State $W \gets W + L^{(B)}_{:,i} \cdot \Delta \cdot (L^{(A)}_{:,j})^T$
  \EndFor
  \EndFor
\end{algorithmic}
Now the inner loop adds outer products of vectors to $W$.
Since a sum of outer products can be described by a matrix multiplication, one can batch all updates in the inner loop together using a single matrix multiplication.
This yields the FastKron-antidiagonal algorithm, described in \cref{alg:fastkronantidiagonal}.

\begin{algorithm}[h]
  \caption{FastKron-antidiagonal (FastKron-naive with parallel processing of anti-diagonals)}
  \label{alg:fastkronantidiagonal}
  \begin{algorithmic}
    \State {\bfseries Input:} $W \in \R^{m \times n}$, $L^{(A)} = \mathrm{LDL}(A^{-1}) \in \R^{n \times n}$, $L^{(B)} = \mathrm{LDL}(B^{-1}) \in \R^{m \times m}$
    \For{$k=2$ {\bfseries to} $m+n$}
    \State $D \gets \mathrm{antidiag}_k(W) \in \R^{m \times n}$
    \State $V \gets \mathrm{round}(D)$
    \State $\Delta \gets V - D$
    \State $W \gets W + L^{(B)} \cdot \Delta \cdot (L^{(A)})^T$
    \EndFor
    \State {\bfseries Output:} $W$ (which satisfies $W \in \Z^{m \times n}$)
  \end{algorithmic}
\end{algorithm}

The iterative version of the YAQA algorithm described by \citet[Appendix~6]{tseng2026modelpreserving} in the LDLQ framework is essentially equivalent to \cref{alg:fastkronantidiagonal}, although they use a different base quantizer and the LDLQ description needs a careful (``iterative'') implementation.
While it requires only $\Oh(m+n)$ sequential steps, note that it still requires $\Oh(m^2 n^2)$ total work, just like the naive \cref{alg:fastkronnaive}. This is because each step costs $\Oh(\min(m,n)^2 (m+n))$ due to two matrix multiplications.

\subsection{FastKron-recursive: Divide-and-Conquer Approach for \texorpdfstring{$\Oh(m n (m+n))$}{O(mn(m+n))} Total Work}
\label{sec:fastkron_recursive}

We would like to further reduce the total cost of the algorithm.
To illustrate the algorithmic idea, in this section we present an algorithm based on FastKron-naive, not using any anti-diagonal batching, which brings down \emph{only} the total algorithmic cost, keeping suboptimal $\Oh(m \cdot n)$ sequential steps.

The main observation is that the outer product update $W \gets W + L_{:,i}^{(B)} \cdot \Delta  \cdot (L_{:,j}^{(A)})^T$ costs $\Oh(m \cdot n)$, but most of the matrix entries that are updated, namely the values in the bottom-right region of the matrix, will only be accessed much later.
One would like to defer these updates, so that later one can perform a batched update of the form $W \gets W + L^{(B)} \cdot \Delta \cdot (L^{(A)})^T$ with $\Delta$ being a \emph{dense matrix}.

A similar idea was already used by \citet{frantar2023optq} in the original GPTQ paper; however, they used a fixed block size for the lazy batch updates.
We will use a block size that depends on the remaining problem size, making the algorithm divide-and-conquer style, and also cache-oblivious.
We solve the first half of the problem recursively, propagate the introduced errors to the second half, and then solve the second half of the problem recursively.
For GPTQ, this yields \cref{alg:gptqrecursive}.

But while for GPTQ this modification does not change the total cost, if we apply the same approach to FastKron-naive in a two-dimensional way as shown in \cref{alg:fastkronrecursive}, it \emph{does} bring the total cost down to just $\Oh(m n (m+n))$, which is essentially the same total cubic cost as that of GPTQ; see \cref{tab:complexity}.

\begin{algorithm}[h]
  \caption{GPTQ-recursive}
  \label{alg:gptqrecursive}
  \begin{algorithmic}
    \State {\bfseries Input:} $W \in \R^{m \times n}$, $L = \mathrm{LDL}(A^{-1}) \in \R^{n \times n}$
    \If{$n=1$}
    \State \Return $\mathrm{round}(W)$ \Comment{base case}
    \Else
    \State $s \gets \lfloor n/2 \rfloor$
    \State $V_{:, 0:s} \gets \text{GPTQ-recursive}(W_{:, 0:s}, L_{0:s, 0:s})$ \Comment{process first half recursively}
    \State $W_{:, s:n} \gets W_{:, s:n} + (V_{:, 0:s} - W_{:, 0:s}) \cdot (L_{s:n,0:s})^T$ \Comment{propagate errors to second half}
    \State $V_{:, s:n} \gets \text{GPTQ-recursive}(W_{:, s:n}, L_{s:n, s:n})$ \Comment{process second half recursively}
    \State \Return $V$
    \EndIf
  \end{algorithmic}
\end{algorithm}

\begin{algorithm}[H]
  \caption{FastKron-recursive (FastKron-naive with recursive divide-and-conquer approach)}
  \label{alg:fastkronrecursive}
  \begin{algorithmic}
    \State {\bfseries Input:} $W \in \R^{m \times n}$, $L^{(A)} = \mathrm{LDL}(A^{-1}) \in \R^{n \times n}$, $L^{(B)} = \mathrm{LDL}(B^{-1}) \in \R^{m \times m}$
    \If{$n=m=1$}
    \State \Return $\mathrm{round}(W)$
    \ElsIf{$m < n$}
    \State $s \gets \lfloor n/2 \rfloor$
    \State $V_{:, 0:s} \gets \text{FastKron-recursive}(W_{:, 0:s}, L^{(A)}_{0:s, 0:s}, L^{(B)})$ \Comment{process first half of columns}
    \State $W_{:, s:n} \gets W_{:, s:n} + L^{(B)} \cdot (V_{:, 0:s} - W_{:, 0:s}) \cdot (L^{(A)}_{s:n, 0:s})^T$ \Comment{propagate error}
    \State $V_{:, s:n} \gets \text{FastKron-recursive}(W_{:, s:n}, L^{(A)}_{s:n, s:n}, L^{(B)})$ \Comment{process second half of columns}
    \State \Return $V$
    \Else
    \State $s \gets \lfloor m/2 \rfloor$
    \State $V_{0:s, :} \gets \text{FastKron-recursive}(W_{0:s, :}, L^{(A)}, L^{(B)}_{0:s, 0:s})$ \Comment{process first half of rows}
    \State $W_{s:m, :} \gets W_{s:m, :} + L^{(B)}_{s:m, 0:s} \cdot (V_{0:s, :} - W_{0:s, :}) \cdot (L^{(A)})^T$ \Comment{propagate error}
    \State $V_{s:m, :} \gets \text{FastKron-recursive}(W_{s:m, :}, L^{(A)}, L^{(B)}_{s:m, s:m})$ \Comment{process second half of rows}
    \State \Return $V$
    \EndIf
  \end{algorithmic}
\end{algorithm}

\subsection{FastKron: Combining the Two Algorithmic Approaches}

Finally, we can combine the two algorithmic techniques from \cref{sec:fastkron_antidiagonal} and \cref{sec:fastkron_recursive}, yielding the final \emph{FastKron} algorithm: We process anti-diagonals in parallel, and we do a recursive divide-and-conquer split over the ``anti-diagonal dimension''.
This means we recursively process the first half of anti-diagonals, then we propagate the error to the second half of anti-diagonals, and finally we recursively process the second half of anti-diagonals.
See \cref{alg:fastkron}.
This unites the best of both worlds: Just $\Oh(m+n)$ sequential steps and just $\Oh(m n (m+n))$ total work; see \cref{tab:complexity}.

\begin{algorithm}[h]
  \caption{FastKron (combination of FastKron-antidiagonal and FastKron-recursive)}
  \label{alg:fastkron}
  \begin{algorithmic}
    \State {\bfseries Input:} $W \in \R^{m \times n}$, $L^{(A)} = \mathrm{LDL}(A^{-1}) \in \R^{n \times n}$, $L^{(B)} = \mathrm{LDL}(B^{-1}) \in \R^{m \times m}$, $k$, $l$
    \If{$l = k+1$} \Comment{$k:l$ is the range of anti-diagonals to process, initially $k=1$ and $l=m+n$}
    \State \Return $\mathrm{round}(\mathrm{antidiag}_l(W))$
    \Else
    \State $s \gets \lfloor (k+l)/2 \rfloor$
    \State $V \gets \text{FastKron}(k, s)$ \Comment{process first half of anti-diagonals}
    \State $W \gets W + \mathrm{antidiag}_{s:l} \left( L^{(B)} \cdot (V - \mathrm{antidiag}_{k:s} (W)) \cdot (L^{(A)})^T \right)$ \Comment{propagate error}
    \State $V \gets V + \text{FastKron}(s, l)$ \Comment{process second half of anti-diagonals}
    \State \Return $V$
    \EndIf
  \end{algorithmic}
\end{algorithm}

The following theorem, proved in \cref{app:equivproofs}, asserts that all the modifications we have done are valid, i.e., they do not change the output of the algorithm.
\cref{app:guarantees} also provides error bounds.

\begin{theorem}
  \label{thm:equiv}
  The algorithms FastKron-naive, FastKron-antidiagonal, FastKron-recursive, and FastKron are all equivalent, i.e., on the same inputs they produce the same outputs.
\end{theorem}

\subsection{Computational Complexity}
\label{sec:complexity}

\begin{table}[h]
    \caption{Asymptotic computational cost of the algorithms for quantizing $W \in \R^{m \times n}$. Lists sequential depth and total cost. Only the core algorithm without Cholesky decomposition is considered.}
    \label{tab:complexity}
    \begin{center}
    \setlength{\tabcolsep}{5pt}
    \begin{tabular}{@{}lccccc@{}}
        \hline
        & & \multicolumn{2}{c}{Sequential Steps} & \multicolumn{2}{c}{Total Cost} \\
        \cmidrule(lr){3-4} \cmidrule(l){5-6}
        Algorithm & Curvature & general & $m \asymp n$ & general & $m \asymp n$ \\
        \hline
        GPTQ                & one-sided & $n$   & $n$   & $m n^2$   & $n^3$ \\
        FastKron-naive (essentially BoA)        & two-sided & $mn$  & $n^2$ & $m^2 n^2$ & $n^4$ \\
        FastKron-antidiagonal (essentially YAQA) & two-sided & $m+n$ & $n$   & $m^2 n^2$ & $n^4$ \\
        FastKron-recursive    & two-sided & $mn$  & $n^2$ & $mn(m+n)$ & $n^3$ \\
        \textbf{FastKron}     & two-sided & $m+n$ & $n$   & $mn(m+n)$ & $n^3$ \\
        \hline
    \end{tabular}
    \end{center}
\end{table}

\cref{tab:complexity} summarizes the number of sequential steps and total cost of all the described algorithms.
These results are proved in \cref{app:complexityproofs}.
We can see that FastKron recovers essentially the same complexity as GPTQ both with respect to parallelizability (linear sequential depth) and total cost (cubic), while handling an arbitrary Kronecker-factored Hessian $A \otimes B$, making the algorithm informed not just by the geometry of input features but also by the geometry of output features.

We note that these costs include neither the Cholesky decomposition, nor the accumulation of the Hessian from the calibration dataset, nor the matrix multiplications needed to send the calibration dataset through the neural network.
The last of these turns out to actually be the dominant cost.

Indeed, suppose we quantize a linear module $W \in \R^{m \times n}$, and we collect the curvature information (Hessian) from a calibration dataset of $k$ sample inputs. For GPTQ, the cost is distributed as follows:
\begin{itemize}
    \item $\Oh(k n^2)$ for accumulating the Hessian $A \otimes I_m$ (where $A = \E[a a^T]$)
    \item $\Oh(n^3)$ for computing the inverse and the Cholesky decomposition of the Hessian
    \item $\Oh(m n^2)$ for the core GPTQ algorithm ($n$ sequential steps with $\Oh(m n)$ operations each)
    \item $\Oh(k m n)$ for (forward) propagating the calibration data through the linear module
\end{itemize}
Note that, since usually $k \gg m$ and $k \gg n$, the first and last terms dominate.
(See \cref{sec:efficientbackprop} for realistic sizes of $m,n,k$.)
For FastKron, similarly, the first and last terms dominate:
\begin{itemize}
    \item $\Oh(k (n^2 + m^2)) = \Oh(k \max(m,n)^2)$ for accumulating the Hessian $A \otimes B$
    \item $\Oh(n^3 + m^3) = \Oh(\max(m,n)^3)$ for computing inverses and Cholesky decompositions
    \item $\Oh(mn (m+n))$ for the core FastKron algorithm (split over $\Oh(m+n)$ sequential steps)
    \item $\Oh(k m n)$ for (forward/backward) propagating the calibration data through the module
\end{itemize}

Thus, in this calibration regime, as with GPTQ, the core FastKron algorithm does not dominate the end-to-end quantization cost.
Instead, the principal costs arise from accumulating the Hessian and propagating the calibration data through the module.

Note that this would \emph{not} be the case without our algorithmic improvement: The previous best implementation of two-sided GPTQ was YAQA (i.e., FastKron-antidiagonal), which has quartic total cost instead of cubic, thus becoming the dominant cost when the size of matrices increases.

\subsection{Benchmarks}
\label{sec:fastkron_benchmarks}

We implemented all algorithms with comparable implementations and benchmarked them on square matrices.
See \cref{tab:benchmarks} and \cref{fig:benchmarks} for the results.
The implementations are described in \cref{app:implementation}.
Please note that the main concern of the implementations is \emph{comparability}; for example, we use base kernels on equally sized blocks/tiles.
Each implementation can be further accelerated with kernel engineering, block size tuning, or even custom memory layout, so we do not claim that these are the fastest implementations, but rather that they are suitable for comparison.

Even if the implementations can still be accelerated by some constant factor, the important takeaway is that FastKron-antidiagonal (i.e. YAQA) becomes impractical for matrices that are used in practice.
For example, square matrices with side length $\approx 10^5$ (or equivalent rectangular matrices) can appear in large Transformer models, but are impractical to quantize with FastKron-antidiagonal in large numbers.
FastKron resolves this issue by reducing the asymptotic cost from quartic to cubic.

\begin{table}[h]
    \caption{Runtime in milliseconds of the relevant described algorithms on square matrices $W \in \R^{n \times n}$, measured on an NVIDIA A40. Only the core algorithm is measured, excluding matrix inversion and Cholesky factorization. FastKron is only a constant factor ($\approx 4\times$) slower than GPTQ-recursive. This is expected and hidden in the $\Oh$-notation. The important part is that it is \emph{only} a constant factor: FastKron has the same asymptotic growth as GPTQ, whereas FastKron-antidiagonal becomes impractical for matrices of larger size.}
    \label{tab:benchmarks}
    \begin{center}
    \setlength{\tabcolsep}{3pt}
    \begin{tabular}{@{}lrrrrrrrr@{}}
        \hline
        $n$ & 1536 & 2048 & 3072 & 4096 & 6144 & 8192 & 12288 & 16384 \\
        \hline
        GPTQ                                     & 4    & 6    & 17   & 38   & 137   & 293   & 1089   & 2285                \\
        GPTQ-recursive                           & 5    & 5    & 7    & 12   & 29    & 59    & 183    & 408                 \\
        FastKron-antidiagonal (ess.\ YAQA)       & 76   & 165  & 595  & 1644 & 7579  & 23084 & 113952 & $>2 \cdot 10^5$   \\
        FastKron-recursive                       & 1260 & 2237 & 5046 & 8960 & 20217 & 35906 & 81118  & 144204              \\
        \textbf{FastKron}                        & 12   & 19   & 33   & 58   & 127   & 250   & 731    & 1546                \\
        \hline
    \end{tabular}
    \end{center}
\end{table}

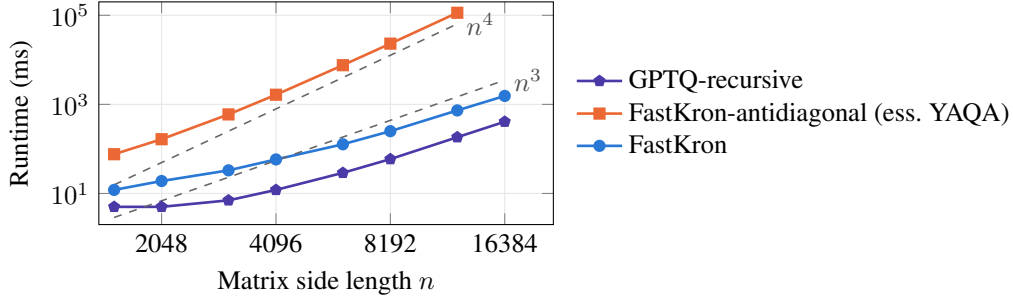
\begin{figure}[h]
  \centering
  \begin{tikzpicture}
    \begin{loglogaxis}[
        width=0.52\linewidth,
        height=0.31\linewidth,
        log basis x=2,
        xmin=1400, xmax=22000,
        ymin=2, ymax=2e5,
        xtick={2048,4096,8192,16384},
        xticklabels={2048,4096,8192,16384},
        xlabel={Matrix side length $n$},
        ylabel={Runtime (ms)},
        grid=major,
        major grid style={black!10},
        legend style={draw=none, at={(1.03,0.5)}, anchor=west},
        legend cell align=left,
        every axis plot/.append style={line width=1pt, mark size=1.8pt},
      ]
      \addplot[colorgptqrecursive, mark=pentagon*] coordinates {
        (1536,5) (2048,5) (3072,7) (4096,12) (6144,29) (8192,59) (12288,183) (16384,408)};
      \addlegendentry{GPTQ-recursive}
      \addplot[colorantidiagonal, mark=square*] coordinates {
        (1536,76) (2048,165) (3072,595) (4096,1644) (6144,7579) (8192,23084) (12288,113952)};
      \addlegendentry{FastKron-antidiagonal (ess.\ YAQA)}
      \addplot[colorfastkron, mark=*] coordinates {
        (1536,12) (2048,19) (3072,33) (4096,58) (6144,127) (8192,250) (12288,731) (16384,1546)};
      \addlegendentry{FastKron}
      \addplot[black!60, dashed, line width=0.6pt, domain=1536:16384, samples=2, forget plot]
        {3500*(x/16384)^3};
      \node[text=black!70, anchor=west] at (axis cs:16384,3500) {$n^3$};
      \addplot[black!60, dashed, line width=0.6pt, domain=1536:12288, samples=2, forget plot]
        {4000*(x/6144)^4};
      \node[text=black!70, anchor=west] at (axis cs:12288,64000) {$n^4$};
    \end{loglogaxis}
  \end{tikzpicture}
  \caption{Benchmarked runtimes on logarithmic axes. One can see that FastKron-antidiagonal grows as $\Oh(n^4)$, while GPTQ-recursive and FastKron only grow as $\Oh(n^3)$ for larger matrices.}
  \label{fig:benchmarks}
\end{figure}

\subsection{Role of \texorpdfstring{$A,B$}{A, B}}
\label{sec:role_ab}

FastKron can be used with any positive (semi-)\footnote{If $A,B$ are not positive definite, they can be regularized by adding a small multiple of the identity matrix.}definite matrices $A \in \R^{n \times n}, B \in \R^{m \times m}$.
$A,B$ inform the algorithm of the parameter curvature, and they are the most critical part of the whole quantization pipeline.
On the one hand, their choice determines the quality of the quantized model.
On the other hand, their computation dominates the cost of the whole quantization pipeline.
For the remainder of the paper, \cref{sec:hessians} outlines various ways to choose $A,B$, and \cref{sec:efficientbackprop} discusses an efficient method to compute the most important (``global'') variants for $A,B$ practically.

\section{Curvature Approximations}
\label{sec:hessians}

The role of $A \otimes B$ is to give a Kronecker-factored approximation of the curvature of the parameter $W$.
Here, curvature refers to the Hessian of some loss with respect to $\vecop(W)$.
In other words, FastKron targets $D_{A,B} = \lVert B^{1/2} (V - W) A^{1/2} \rVert_F^2$, and we want this to be a second-order approximation of some loss that we care about.
It only remains to decide:
\begin{enumerate}
\item Which loss to consider for $W$.
\item How to approximate the resulting Hessian by a Kronecker product.
\end{enumerate}
These two questions are discussed briefly in the next two subsections, and in detail in \cref{app:hessians}.

For the loss, we list two possibilities:
The first is the global negative log-likelihood (or any loss of the network outputs followed by a Gauss--Newton approximation of the Hessian).
The second is a local loss function based on the surrounding modules of the linear module $W$.
See \cref{sec:hessianloss}.

For the Kronecker factorization, we list three methods from the literature:
One is based on an independence assumption (``Marginal''), another is based on an alternating power iteration scheme (``Frobenius''), and a third is an alternative iteration scheme (``FlipFlop'').
See \cref{sec:hessianfactor}.

This gives six possibilities for the Kronecker-factored curvature approximations in FastKron.
\citet{kim2025boa} have quantized and evaluated models using the local attention module Hessian (BoA).
\citet{tseng2026modelpreserving} evaluated the global Hessian with Frobenius approximation (YAQA).
\citet{birnick2026waterkron} evaluated all listed Hessian approximations (Marginal, Frobenius, FlipFlop) for the global Hessian (WaterKron).
While these papers give the empirical validation for using the two-sided GPTQ algorithm, this paper focuses on the algorithmic efficiency of the algorithm, and therefore intentionally leaves out empirical evaluation beyond benchmarks.

\subsection{Choosing the Loss: Local vs.\ Global}
\label{sec:hessianloss}

Denote the input to the linear unit $W$ by $a$ and the output by $z$, so that $z = W a$.
If we write $\cL_{x,y}(W) := -\log p_{f_W}(y \mid x)$ for the negative log-likelihood of a sample $(x,y)$ and $b := \nabla_z \cL_{x,y}$, then
\(\nabla_{\vecop(W)} \cL_{x,y} = a \otimes b\).
The Hessian of the expected cross entropy loss of the network with respect to its own predictive distribution is the Fisher information matrix:
\begin{equation*}
    H_{\mathrm{FIM}} = \E_{x,y \sim f_W(x)} [\nabla^2_{\vecop(W)} \cL_{x,y}]
    =
    \E_{x,y \sim f_W(x)}
    \left[
        (a \otimes b)(a \otimes b)^T
    \right]
    =
    \E_{x,y \sim f_W(x)}
    \left[
        aa^T \otimes bb^T
    \right]
\end{equation*}
This is the most natural curvature to consider; however, it comes at the cost of backpropagation through the whole network.
If done naively, this can incur high computational and memory requirements.
See \cref{sec:efficientbackprop} for an improved method of computing with this Hessian.

Alternatively, one can make a local curvature approximation that only requires data from the surrounding modules.
This approach is used by the BoA algorithm \citep{kim2025boa}.
They consider a local loss for the self-attention module of a Transformer and derive the Hessians:
\begin{align*}
H_{W_Q} &= \E[a a^T \otimes k k^T] &\text{where} \; k &= W_K a \;\;\text{and $a \in \R^{n \times N}$ represents a sequence of length $N$} \\
H_{W_K} &= \E[a a^T \otimes q q^T] &\text{where} \; q &= W_Q a \;\;\text{and $a \in \R^{n \times N}$ represents a sequence of length $N$}
\end{align*}
We derive similar local Hessians for the MLP module of a Transformer (see \cref{app:hessianloss}):
\begin{align*}
H_{W_\mathrm{up}} &= \E[a a^T \otimes (c_\mathrm{up} c_\mathrm{up}^T \odot W_\mathrm{down}^T W_\mathrm{down})] &\text{where} \; c_\mathrm{up} &= \sigma (W_\mathrm{gate} a) \\
H_{W_\mathrm{gate}} &= \E[a a^T \otimes (c_\mathrm{gate} c_\mathrm{gate}^T \odot W_\mathrm{down}^T W_\mathrm{down})] &\text{where} \; c_\mathrm{gate} &= \sigma' (W_\mathrm{gate} a) \odot W_\mathrm{up} a
\end{align*}

\subsection{Choosing a Kronecker Approximation}
\label{sec:hessianfactor}

Once we have a Hessian of the form $H = \E[a a^T \otimes b b^T]$, we still have to find a Kronecker-factored approximation $H \approx A \otimes B$ of it.
We list three methods of doing that, with details in \cref{app:hessianfactor}.

The first (``Marginal'') makes an independence assumption to pull the expectation out, yielding:
\begin{equation*}
    H \approx \E[a a^T] \otimes \E[b b^T] \qquad\quad \text{i.e.} \qquad\quad A := \E[a a^T] \qquad B := \E[b b^T]
\end{equation*}
The second (``Frobenius'') targets $\lVert H - A \otimes B \rVert_F$ by using an alternating power iteration:
\begin{equation*}
A_{t+1} \propto \E[\tr(b^T B_t b) \cdot a a^T] \qquad B_{t+1} \propto \E[\tr(a^T A_{t+1} a) \cdot b b^T].
\end{equation*}
This Kronecker approximation method has been used by \citet{tseng2026modelpreserving} (YAQA).
A third method (``FlipFlop'') was recently proposed by \citet{birnick2026waterkron}.
It targets $\Phi(H, A \otimes B) := \frac{1}{mn} \tr(H (A \otimes B)^{-1})/(|H|^{1/mn} \cdot |A \otimes B|^{-1/mn})$, which is an objective specifically to minimize quantization distortion measured in the full curvature $H$, and it uses a similar alternating iteration:
\begin{equation*}
A_{t+1} \propto \E[\tr(b^T B_t^{-1} b) \cdot a a^T] \qquad B_{t+1} \propto \E[\tr(a^T A_{t+1}^{-1} a) \cdot b b^T].
\end{equation*}

\section{Computing Backpropagated Hessians Efficiently}
\label{sec:efficientbackprop}

In this section, we explain an improved quantization flow to compute Kronecker approximations of the ``global'' Hessian $H_{\mathrm{FIM}}$ from \cref{sec:hessianloss} efficiently.
Note that for any of the approximation methods from \cref{sec:hessianfactor}, all we need is to be able to sample the module input $a \in \R^n$ and the module output gradient $b \in \R^m$ according to the particular distribution discussed in \cref{sec:hessianloss}.

In practice, the distribution of network inputs $x$ is taken from a calibration dataset.
Accordingly, $a$ is computed by forward-propagating $x$ to the module $W$, and $b$ is computed by forward-propagating $a$ further till the end of the network, then sampling a label $y$, computing the negative log-likelihood, and back-propagating the gradient to the module $W$.

However, the concrete schedule of doing this, the \emph{quantization flow}, determines the compute and memory requirements of this endeavor.
We first review the classic forward flow and two naive flows for our problem.
Then we propose our new flow.
We discuss compute and memory requirements for each flow under a simplified network with $\ell$ layers, each of which contains one $n \times n$ linear module, and with a calibration dataset of $k$ inputs.
We only consider the cost of forward/backward-propagating activations/gradients and executing the Hessian approximation method, since this dominates the cost of running the actual FastKron rounding algorithm when $k \gg n$, as discussed in \cref{sec:complexity}.
Concretely, we assume $k \geq n \ell$.
Practical numbers are usually $k \geq 256n$, $\ell \leq 256$, and $2^9 \leq n \leq 2^{14}$; in particular, this simplifying assumption is satisfied.
``Fast memory'' in practice is VRAM of the GPU.
``Slow memory'' can be RAM of the host machine, or might even be (partially) offloaded to much slower storage.
All costs are summarized in \cref{tab:backprop}.

\textbf{Forward Flow.}
We first recall the classic GPTQ flow, which also works for the ``local'' Hessian variants of FastKron.
We proceed layer by layer, and only load one layer at a time into memory.
The computational cost is $\Oh(k n^2)$ in each layer to forward the activations, so $\Oh(k n^2 \ell)$ in total.
For each layer, we need to keep $W, A, B$ in fast memory, costing $\Oh(n^2)$ fast memory.
We also need to store the activations $a$ for all $k$ calibration inputs, costing $\Oh(kn)$, but this can reside in slow memory.

\textbf{Fresh Backward Passes Flow.}
Recall from above that for the ``global'' Hessian, we additionally need to send activations through the whole network and back-propagate the gradient.
The naive way of doing this would be to keep proceeding layer by layer, and do a fresh backward pass for all $k$ calibration samples each time.
The computational cost of this, however, is quadratic in the number of layers, namely $\Oh(k n^2 \ell^2)$.
The fast and slow memory requirements are $\Oh(n^2)$ and $\Oh(kn)$, respectively, as in the forward flow (provided batch sizes for backward passes are kept small enough).

\textbf{Simultaneous Hessian Flow.}
To bypass that computational burden, one can just propagate each of the $k$ calibration samples only \emph{once} forward and backward through the network, and compute/accumulate the Hessian approximations for all modules at the same time.
This is the flow used by \citet{tseng2026modelpreserving} for YAQA.
However, while it brings the computational cost back to $\Oh(k n^2 \ell)$, it introduces massive requirements for fast memory: We need to have all $\Oh(\ell)$ Hessian Kronecker factors readily available, necessitating $\Oh(n^2 \ell)$ fast memory.
This is at least the size of the whole model, but $\Oh$-notation and model simplification hide large constants: in practice, it usually requires at least 5 times the model size of fast memory.

\textbf{Recursive Halving Flow.}
This leads us to propose a the recursive divide-and-conquer flow which aims to find a sweet spot: low fast memory requirements and reasonable compute requirements.
It does so by employing recursive checkpointing.
Given a range of layers, initially the full list of layers, the recursive algorithm requires the input to the first layer and the gradient of the output of the last layer.
Next, a forward and backward pass (for all $k$ calibration samples) is required to compute the activations as well as their gradient at exactly the middle of the layer list (``mid-layer activations'').
Then one can split the layer list into the first half and second half.
For each of those lists one has the input as well as the gradient of the output, so one can call the algorithm recursively on the two layer lists.
When the list consists of just a single layer, one has reached the base case, where one can use one of the Kronecker approximation methods and call FastKron.
This divide-and-conquer approach requires $\Oh(kn^2 \ell \log \ell)$ compute, $\Oh(n^2)$ fast memory, and $\Oh(kn \log \ell)$ slow memory.
Thus, compared to the forward flow, it only adds a factor of $\log \ell$ to compute and slow memory, while providing each module with additional global back-propagated curvature information.

Without this flow improvement it can be impractical to quantize larger models in environments with limited compute and fast memory capabilities.
This (affordability) is also the reason why we do not provide comparative empirical benchmarks for the flows.
For example, on a single GPU, the ``Fresh Backward Passes'' and ``Simultaneous Hessian'' flows become impractical for models beyond a few billion parameters.

\begin{table}[H]
    \caption{Asymptotic compute and memory costs of the different quantization flows under $k$ calibration samples, $\ell$ layers, modules of size $n \times n$, and the realistic simplifying assumption $k \geq n \ell$.}
    \label{tab:backprop}
    \begin{center}
    \begin{tabular}{@{}llccc@{}}
        \hline
        Flow & Hessian & Compute & Fast mem. & Slow mem. \\
        \hline
        Forward & local & $k n^2 \ell$ & $n^2$ & $k n$ \\
        Fresh Backward Passes & global & $k n^2 \ell^2$ & $n^2$ & $k n$ \\
        Simultaneous Hessian (YAQA) & global & $k n^2 \ell$ & $n^2 \ell$ & -- \\
        \textbf{Recursive Halving} & global & $k n^2 \ell \log \ell$ & $n^2$ & $k n \log \ell$ \\
        \hline
    \end{tabular}
    \end{center}
\end{table}

\section*{Acknowledgments}

We gratefully acknowledge partial support by the National Science Foundation, via the DMS-2410717 grant.

\bibliography{fastkron}

@inproceedings{frantar2023optq,
  title={{OPTQ}: Accurate quantization for generative pre-trained transformers},
  author={Frantar, Elias and Ashkboos, Saleh and Hoefler, Torsten and Alistarh, Dan},
  booktitle={The Eleventh International Conference on Learning Representations},
  year={2023}
}

@inproceedings{birnick2026lattice,
  title={The Lattice Geometry of Neural Network Quantization: A Short Equivalence Proof of {GPTQ} and {B}abai's Algorithm},
  author={Birnick, Johann},
  booktitle={International Conference on Learning Representations},
  volume={2026},
  pages={101779--101787},
  year={2026}
}

@inproceedings{chen2026geometry,
  title={The geometry of {LLM} quantization: {GPTQ} as {B}abai's nearest plane algorithm},
  author={Chen, Jiale and Shabanzadeh, Yalda and Crn{\v{c}}evi{\'c}, Elvir and Hoefler, Torsten and Alistarh, Dan},
  booktitle={International Conference on Learning Representations},
  volume={2026},
  pages={122653--122695},
  year={2026}
}

@article{zhang2025provable,
  title={Provable post-training quantization: Theoretical analysis of {OPTQ} and {Qronos}},
  author={Zhang, Haoyu and Zhang, Shihao and Colbert, Ian and Saab, Rayan},
  journal={arXiv preprint arXiv:2508.04853},
  year={2025}
}

@inproceedings{kim2025boa,
  title={{BoA}: Attention-aware Post-training Quantization without Backpropagation},
  author={Junhan Kim and Kim, Ho-young and Eulrang Cho and Chungman Lee and Joonyoung Kim and Yongkweon Jeon},
  booktitle={Forty-second International Conference on Machine Learning},
  year={2025},
  url={https://openreview.net/forum?id=Uvj6XcSJ5d}
}

@inproceedings{tseng2026modelpreserving,
  title={Model-Preserving Adaptive Rounding},
  author={Albert Tseng and Zhaofeng Sun and De Sa, Christopher},
  booktitle={Forty-third International Conference on Machine Learning},
  year={2026},
  url={https://openreview.net/forum?id=PKFilPWjMI}
}

@incollection{van1993approximation,
  title={Approximation with {K}ronecker products},
  author={Van Loan, Charles F. and Pitsianis, Nikos},
  booktitle={Linear Algebra for Large Scale and Real-Time Applications},
  pages={293--314},
  year={1993},
  publisher={Springer}
}

@article{martens2020new,
  title={New insights and perspectives on the natural gradient method},
  author={Martens, James},
  journal={Journal of Machine Learning Research},
  volume={21},
  number={146},
  pages={1--76},
  year={2020}
}

@article{chen2026gptq2d,
  title={{GPTQ-2D}: Cubic-Time Two-Sided Adaptive Rounding},
  author={Chen, Jiale and Hoefler, Torsten and Alistarh, Dan},
  journal={arXiv preprint arXiv:2607.27042},
  year={2026}
}

@article{birnick2026waterkron,
  title={{WaterKron} and {FlipFlop} {H}essian: Information-Theoretically Grounded Quantization with {K}ronecker-factored {H}essians},
  author={Birnick, Johann and Saab, Rayan},
  journal={arXiv preprint arXiv:2609.14706},
  year={2026}
}

@inproceedings{tillet2019triton,
  title={Triton: An Intermediate Language and Compiler for Tiled Neural Network Computations},
  author={Tillet, Philippe and Kung, Hsiang-Tsung and Cox, David},
  booktitle={Proceedings of the 3rd ACM SIGPLAN International Workshop on Machine Learning and Programming Languages},
  pages={10--19},
  year={2019}
}
\bibliographystyle{plainnat}

\newpage
\appendix
\crefalias{section}{appendix}

\section{Equivalence Proof}
\label{app:equivproofs}

\begin{proof}[Proof of \cref{thm:equiv}]
Recall that $L^{(A)} = \mathrm{LDL}(A^{-1})$ and $L^{(B)} = \mathrm{LDL}(B^{-1})$ are lower triangular and normalized to satisfy $L^{(A)}_{i,i} = L^{(B)}_{j,j} = 1$.
Define the following partial order:
\begin{equation*}
    (i,j) \preceq (p,q)
    \quad:\Longleftrightarrow\quad
    i \leq p \ \text{and}\ j \leq q
\end{equation*}
A rounding error $\Delta_{i,j}$ adds $L^{(B)}_{p,i} \cdot \Delta_{i,j} \cdot L^{(A)}_{q,j}$ to entry $W_{p,q}$.
This vanishes unless $(i,j) \preceq (p,q)$.
Hence an entry can be rounded as soon as it has received all updates from the entries with smaller indices than itself with respect to the order $\preceq$.

Therefore, one may reorder the rounding steps, provided that the visited sequence of entries does not decrease with respect to $\preceq$.
One may also delay the updates caused by rounding, provided that before rounding an entry, the contributions from all smaller indices with respect to $\preceq$ have been accumulated.
Lastly, because of the simple algebraic fact that a sum of outer products corresponds to a matrix multiplication, one may batch updates using matrix multiplication.
Concretely, for any matrix $\Delta$ of rounding errors we have:
\begin{equation*}
    L^{(B)} \cdot \Delta \cdot (L^{(A)})^T
    = \sum_{i,j} L^{(B)}_{:,i} \cdot \Delta_{i,j} \cdot (L^{(A)}_{:,j})^T
\end{equation*}

FastKron-naive visits entries in column-major order, which respects $\preceq$, and does not delay or batch any updates.
FastKron-antidiagonal batches updates from each anti-diagonal, but does not delay them, and the anti-diagonal schedule respects $\preceq$ because if the anti-diagonal indices of entries $(i,j)$ and $(p,q)$ compare as $i + j \geq p + q$, then the only way for $(i,j) \preceq (p,q)$ to hold is $(i,j) = (p,q)$.
The divide-and-conquer algorithms FastKron-recursive and FastKron split the indices into two sets $S_1$ and $S_2$, and it holds $s_1 \preceq s_2$ for all $s_1 \in S_1, s_2 \in S_2$, and all the updates to entries in $S_2$ caused by entries in $S_1$ are accumulated before any entry of $S_2$ is rounded.
\end{proof}

\section{Error Bound Guarantees}
\label{app:guarantees}

We recall the error bound from GPTQ as proved recently in the works by \citet{zhang2025provable, birnick2026lattice, chen2026geometry}.
For the statements, we use the reverse LDL decomposition: For a positive definite matrix $A$, when we write $A = U D U^T$, then $U$ is the unique upper-triangular matrix with unit diagonal $U_{i,i} = 1$ that satisfies this equation with a diagonal matrix $D$.

\begin{theorem}
  \label{thm:gptqbound}
  Suppose \cref{alg:gptq} outputs $V \in \Z^{m \times n}$ on inputs $W \in \R^{m \times n}$ and $L = \mathrm{LDL}(A^{-1}) \in \R^{n \times n}$, and write $A = U D U^T$.
  Then
  \begin{equation*}
    D_A = \lVert (V - W) A^{1/2} \rVert_F^2 \leq \frac{m}{4} \sum_{j=1}^n D_{j,j} \leq \frac{m}{4} \sum_{j=1}^n A_{j,j} = \frac{m}{4} \tr(A).
  \end{equation*}
\end{theorem}

From this we can deduce an error bound for FastKron.

\begin{corollary}
  \label{thm:fastkronbound}
  Suppose \cref{alg:fastkron} outputs $V \in \Z^{m \times n}$ on inputs $W \in \R^{m \times n}$, $L^{(A)} = \mathrm{LDL}(A^{-1}) \in \R^{n \times n}$, and $L^{(B)} = \mathrm{LDL}(B^{-1}) \in \R^{m \times m}$, and write $A = U^{(A)} D^{(A)} (U^{(A)})^T$ and $B = U^{(B)} D^{(B)} (U^{(B)})^T$.
  Then
  \begin{align*}
    D_{A,B} = \lVert B^{1/2} (V - W) A^{1/2} \rVert_F^2 &\leq \frac{1}{4} \sum_{i=1}^m \sum_{j=1}^n D^{(B)}_{i,i} D^{(A)}_{j,j} \leq \frac{1}{4} \tr(A) \cdot \tr(B).
  \end{align*}
\end{corollary}

\begin{proof}
By \cref{thm:equiv}, it suffices to prove this for the output of \cref{alg:fastkronnaive}.
But the output of \cref{alg:fastkronnaive} is just the output of \cref{alg:gptq} when called with inputs $\vecop(W)^T \in \R^{1 \times mn}$ and $\mathrm{LDL}((A \otimes B)^{-1}) = L^{(A)} \otimes L^{(B)}$.
Now \cref{thm:gptqbound} with a single row, using $(A \otimes B)^{1/2} = A^{1/2} \otimes B^{1/2}$ and $A \otimes B = (U^{(A)} \otimes U^{(B)}) (D^{(A)} \otimes D^{(B)}) (U^{(A)} \otimes U^{(B)})^T$, guarantees that
\begin{align*}
  \lVert (A^{1/2} \otimes B^{1/2}) \cdot \vecop(V - W) \rVert_2^2
  &\leq \frac{1}{4} \sum_{k=1}^{m \cdot n} (D^{(A)} \otimes D^{(B)})_{k,k}
  = \frac{1}{4} \sum_{i=1}^m \sum_{j=1}^n D^{(B)}_{i,i} D^{(A)}_{j,j} \\
  &\leq \frac{1}{4} \sum_{k=1}^{m \cdot n} (A \otimes B)_{k,k}
  = \frac{1}{4} \tr(A) \cdot \tr(B).
\end{align*}
We finish by applying \cref{eq:vec}.
\end{proof}

These guarantees are stated for simplicity in our convention that $V \in \Z^{m \times n}$, but they also apply (with appropriate scaling modification) when rows and columns of $V$ are additionally scaled by real-valued scaling factors.
There are also expected error guarantees for the case when the entries $W_{i,j}$ are iid Gaussian, and also measured directly with respect to the full Hessian $H$ (instead of the approximation $A \otimes B$); see \citet{birnick2026waterkron}.

\section{Complexity Proofs}
\label{app:complexityproofs}

In this section, we prove the bounds on sequential steps and total cost
reported in \cref{tab:complexity}.

We consider only the core algorithms, excluding the matrix inversion or Cholesky factorization.
Total cost counts scalar arithmetic operations.
Sequential steps count parallel scalar operations as a single step, and also count a full matrix multiplication as a single step (which is a simplified complexity model since technically a matrix multiplication needs log-many sequential steps).

\textbf{Sequential Steps.}
GPTQ, FastKron-naive, and FastKron-antidiagonal have $n$, $mn$, and $m+n-1$ iterations, respectively, where each iteration only requires a constant number of sequential steps.
The divide-and-conquer algorithms FastKron-recursive and FastKron process their two children one after another, and otherwise only use a constant number of sequential steps, thus we need to count the number of nodes in the binary recursion tree, or equivalently the number of leaves.
FastKron-recursive has one leaf for each of the $mn$ matrix entries, while FastKron has one leaf for each of the $m+n-1$ anti-diagonals.

\textbf{GPTQ Total Cost.}
The dominant cost in each of the $n$ iterations is to update $\Oh(m n)$ entries of $W$, giving $\Oh(m n^2)$ total cost.

\textbf{FastKron-naive Total Cost.}
Similarly, each iteration costs $\Oh(m n)$ for the outer product update of $W$, but there are $m n$ iterations, giving $\Oh(m^2 n^2)$.

\textbf{FastKron-antidiagonal Total Cost.}
Let $d_k$ be the number of entries of the $k$\textsuperscript{th} anti-diagonal.
Its update is a sum of $d_k$ outer products, which can be evaluated as the product of an $m \times d_k$ matrix with a $d_k \times n$ matrix, costing $\Oh(m n d_k)$.
Since every entry is on exactly one anti-diagonal we have $\sum_{k=2}^{m+n} d_k = mn$, hence the total work is:
\begin{equation*}
    \sum_{k=2}^{m+n} \Oh(mn d_k)
    = \Oh\left(mn \sum_{k=2}^{m+n} d_k\right)
    = \Oh(m^2n^2)
\end{equation*}

\textbf{FastKron-recursive Total Cost.}
In each node of the recursion tree, operating on a matrix of size $r \times s$, the cost of error propagation is $\Oh(r^2 s + r s^2) = \Oh(rs (r+s))$ due to two matrix multiplications.
(Left-multiplication with $L^{(B)}$ and right-multiplication with $(L^{(A)})^T$.)
Distribute this cost equally among the $r s$ entries of the $r \times s$ block, so that each entry pays $\Oh(r+s)$ cost in this recursion node.
Now compute the cost that each entry of the whole $m \times n$ matrix pays over all the recursion blocks that it is contained in:
Each entry appears once in the root node, paying $\Oh(m + n)$.
Then it recursively appears in one of the child blocks, whose number of entries is halved, so that in sum over all those nodes the entry pays $\Oh(m + n) + \Oh((m+n)/2) + \Oh((m+n)/4) + \ldots + \Oh(1)$, which is $\Oh(m+n)$.
Summing this cost over all $mn$ entries, we have a total cost of $\Oh(mn (m+n))$.

\textbf{FastKron Total Cost.}
Similarly, consider a recursion node which operates on the anti-diagonal range $k:l$, and denote the set of entries that lie on these anti-diagonals by $S$.
We claim that the cost of error propagation in this node is $\Oh(|S| \cdot (l-k))$, which when distributed equally among all the $|S|$ entries of this anti-diagonal range gives a cost of $\Oh(l-k)$ per entry.
Now, again, compute the cost that each entry of the whole $m \times n$ matrix pays over all the recursion nodes that it is contained in:
Each entry pays $\Oh(m + n)$ in the root node with range $1:m+n$, and then recursively the number of anti-diagonals is halved for each child, giving also $\Oh(m + n)$ cost for the entry over all recursion nodes.
Summing over all $mn$ entries gives a total cost of $\Oh(mn (m+n))$. \\
It only remains to prove our claim that error propagation in the $k:l$ node costs $\Oh(|S| \cdot (l-k))$.
We have to compute the matrix product $L^{(B)} \Delta (L^{(A)})^T$, but only on the anti-diagonal range $s:l$, and $\Delta$ is supported only on the anti-diagonal range $k:s$, where $s = \lfloor (k+l)/2 \rfloor$ is the range midpoint.
For this, it suffices to compute $E = L^{(B)} \Delta$ only on the entries in $S$, yielding $E'$, and then to compute $E' (L^{(A)})^T$ only on the entries of anti-diagonal range $s:l$.
Indeed, suppose we want to compute the $(i,j)$ entry of $E (L^{(A)})^T$, where $i + j \leq l$.
This is $\sum_p E_{i,p} L^{(A)}_{j,p}$.
Since $L^{(A)}_{j,p} = 0$ for $j < p$, we only need $E_{i,p}$ for $p \leq j$.
Since then $i + p \leq i + j \leq l$, we only need entries of $E$ on the anti-diagonals up to $l$.
Moreover, $E_{i,p} = \sum_{i'} L^{(B)}_{i,i'} \Delta_{i',p}$ vanishes for $i + p \leq k$ since $L^{(B)}_{i,i'} = 0$ for $i < i'$ and $\Delta_{i',p} = 0$ for $i' + p \leq k$.
Hence all needed entries of $E$ lie in $S$.
Both the computation of $E'$ and of the relevant entries of $E' (L^{(A)})^T$ require at most $|S|$ inner products of cost $\Oh(l-k)$ each, landing at $\Oh(|S| \cdot (l - k))$ cost for the error propagation.

\section{Implementation}
\label{app:implementation}

This section describes the implementations whose runtimes were reported in \cref{sec:fastkron_benchmarks}.
The goal of the implementations is not to be the fastest possible, but to be reasonably fast while being comparable for the benchmark results.

All algorithms are implemented using Triton \citep{tillet2019triton}.
They use comparable base kernels on $16 \times 16$ blocks, comparable error propagation, and the same basic row-major memory layout.
Everything runs in full IEEE float32 precision, without use of TF32 tensor cores.

\textbf{Base Kernels.}
Each algorithm processes a $16 \times 16$ tile of $W$ with a base kernel, so that on the order of 16 sequential steps are processed within a single kernel.
The reason for this ``tiled'' schedule choice is that the cost would be dominated by kernel launch overhead if a separate kernel were launched for each sequential step.
The algorithm is executed for each tile in registers, propagating errors within the tile without delaying any updates.
For example, the GPTQ base kernel rounds a tile column by column (16 in total), and FastKron-antidiagonal and FastKron round the 31 anti-diagonals of the tile sequentially.
For GPTQ, all tiles in a band of 16 columns are run in parallel, while for FastKron-antidiagonal and FastKron, all tiles lying on an ``anti-diagonal of tiles'' are run in parallel.
For FastKron-recursive, the base kernel rounds single entries of the tile at a time, in order of the recursion, and runs no tiles in parallel.

\textbf{Error Propagation.}
Between the (potentially parallel) runs of the base kernels, the errors are propagated to other tiles, following the algorithmic descriptions of the algorithms.
For example, GPTQ propagates to all later tiles, while GPTQ-recursive has a recursive error propagation schedule which only propagates errors to the ``next few'' tiles, where ``few'' can also be as much as half of the matrix in the root of the schedule.
The update matrix $L^{(B)} \Delta (L^{(A)})^T$ is computed with two masked matrix multiplications, as in the proof of \cref{app:complexityproofs}: first $E = L^{(B)} \Delta$ is computed on the relevant entries, yielding $E'$, and then $E' (L^{(A)})^T$ is computed on the relevant entries (the ones that are ultimately updated).
For GPTQ we have $B = I_m$, so only the second product is computed.
The matrix multiplications use $128 \times 128$ output blocks if the updated region spans at least 128 rows, columns or anti-diagonals, and $16 \times 16$ blocks otherwise.
(The reason for using larger blocks if possible is to match the complexity model from \cref{app:complexityproofs} where one matrix multiplication counts as a single sequential step.)

\section{Curvature Approximations in Detail}
\label{app:hessians}

Let us first explain the way we use the term \emph{curvature} matrix or \emph{Hessian}.
We use the term curvature in an optimization sense, as in \citet{martens2020new}, which should not be confused with the curvature tensor from Riemannian geometry; instead, what we call curvature is called the \emph{metric} in Riemannian geometry.
Concretely, for us, a curvature matrix is any positive definite matrix which gives a local quadratic approximation of some loss function, which is why we also refer to curvature matrices purely as Hessians (implicitly assuming that they are positive definite).
Thus a curvature matrix captures the local geometry of the loss function; it can also be interpreted as the Gram matrix of an inner product.
The loss function, for the purpose of this paper, measures the distortion between the original and the quantized model.

$H \in \R^{mn \times mn}$ denotes the induced curvature / Gram matrix / Hessian for the (vectorized) weight space $\vecop(W) \in \R^{mn}$.
Since it is computationally infeasible to deal with this directly for the purpose of this paper, we seek to approximate it with a Kronecker product $H \approx A \otimes B$ with $A \in \R^{n \times n}$ and $B \in \R^{m \times m}$.
Then, in a certain sense, $A$ captures the geometry of the input activations, while $B$ captures the geometry of the output gradients.
More concretely, $A$ defines an inner product on row vectors $\delta w \in \R^{1 \times n}$, i.e., on input activation \emph{dual} space, while $B$ defines an inner product on column vectors $\delta z \in \R^{m \times 1}$, i.e., on output activation space.
The former measures the change in loss that $\delta w$ causes, while the latter measures the change in loss that $\delta z$ causes.

We now describe the Hessian constructions from \cref{sec:hessians} in detail.
Recall that we first have to choose a loss, which determines $H$, and then we can choose a Kronecker approximation method to find $A,B$.

\subsection{Choosing the Loss: Local vs.\ Global}
\label{app:hessianloss}

The most desirable way to measure the distortion is to use the loss function $\cL(W)$ which was also used to train the network.
Since the network is given to us in trained state, it is natural to assume the network parameters are at (or near) a local minimizer of $\cL$, so that $\nabla_{\vecop(W)} \cL \approx 0$ and $\nabla^2_{\vecop(W)} \cL$ is positive (semi-)definite.
Keeping only the quadratic term in the Taylor expansion of $\cL$ around $W$ then yields a quadratic loss for $\vecop(V) - \vecop(W)$ with Hessian $H = \nabla^2_{\vecop(W)} \cL$.

Now suppose the output $f_W (x)$ of the neural network represents a probability distribution for the label $y$.
Since the quantized model should approximate the original model, a natural choice is the cross entropy loss with respect to the model's own predictive distribution. That is, the expectation below is taken over the calibration input data and over labels sampled from the predictive distribution of the original model:
\begin{equation*}
    \cL (V) = \E_{x,y \sim f_W(x)} [-\log p_{f_V} (y \mid x)]
\end{equation*}
In this case, it is well known \citep{martens2020new} that the Hessian at $V=W$ is equal to an expectation over gradient outer products, namely it is equal to the Fisher information matrix (FIM):
\begin{equation*}
    \nabla^2_{\vecop(W)} \cL = \E_{x,y \sim f_W(x)} [\nabla_{\vecop(W)} \log p_{f_W} (y \mid x) \nabla_{\vecop(W)} \log p_{f_W} (y \mid x)^T ]
\end{equation*}

Since $W$ is a matrix, this expression has additional structure.
Denote the input to the linear unit $W$ by $a$ and the output by $z$, so that $z = W a$.
If we write $\cL_{x,y}(W) := -\log p_{f_W}(y \mid x)$ and $b := \nabla_z \cL_{x,y}$, then
\(\nabla_{\vecop(W)} \cL_{x,y} = a \otimes b\).
Thus the Fisher information matrix associated with $W$ is:
\begin{equation*}
    H_{\mathrm{FIM}}
    =
    \E_{x,y \sim f_W(x)}
    \left[
        (a \otimes b)(a \otimes b)^T
    \right]
    =
    \E_{x,y \sim f_W(x)}
    \left[
        aa^T \otimes bb^T
    \right]
\end{equation*}

One could instead take the label $y$ from the calibration dataset, rather than sampling it
from the model's predictive distribution. In this case, the corresponding
matrix is known as the \emph{empirical} Fisher information matrix.
For FastKron, this choice would make the quantization more strongly adapted (``overfitted'') to the calibration dataset, which may or may not be desired.
More generally, the Gauss-Newton component of the Hessian for \emph{any} (convex) network loss has the same form as $H_{\mathrm{FIM}}$ above.

Another option is to make a local curvature approximation, which only requires data from surrounding modules; in particular, it does
not require a full backward pass through the network.
This approach is used by the BoA algorithm \citep{kim2025boa}.
BoA considers the self-attention module in Transformer-based architectures, where the attention matrix is computed as $a_{:,i}^T W_Q^T W_K a_{:,j}$.
Here, only for this paragraph, $a$ denotes not just a single token vector but a matrix that holds a full sequence of tokens.
If one uses a Frobenius loss for the attention matrix alone, then the resulting local Hessians are
\begin{align*}
H_{W_Q} &= \E[a a^T \otimes k k^T] &\text{where} \; k &= W_K a \\
H_{W_K} &= \E[a a^T \otimes q q^T] &\text{where} \; q &= W_Q a
\end{align*}
where the expectation is over sequences.
The BoA algorithm forms these local Hessians separately for each attention head,
which allows the attention heads to be processed in parallel.
This parallelism is key to making their method practical.
Indeed, BoA applies GPTQ-style updates in the vectorized weight domain, but does not use any of the algorithmic performance improvements that we discussed in this paper, i.e., they use FastKron-naive.
FastKron therefore provides a direct acceleration of BoA.

The approach of \citet{kim2025boa} does not consider the MLP module in the Transformer architecture.
Matrices in this module do not naturally split into a collection of smaller matrices like the attention heads in the attention module.
Thus, applying FastKron-naive would require $m \cdot n$ sequential steps in the MLP module, making the approach impractical for large weight matrices.
FastKron makes this computation feasible, so we derive local Hessians for the MLP module.
Concretely, we focus on the case of a Gated Linear Unit (GLU) followed by another linear unit, as employed by many modern Transformer-based architectures. For simplicity of exposition we ignore biases.
Fix an activation function $\sigma$. Then the MLP is given by:
\begin{equation*}
    \mathrm{MLP}(a) = W_\mathrm{down} ( W_\mathrm{up} a \odot \sigma(W_\mathrm{gate} a))
\end{equation*}
If we use $D_s$ to denote a diagonal matrix with the vector $s$ on the diagonal, this can be written as:
\begin{equation*}
    \mathrm{MLP}(a) = W_\mathrm{down} D_{\sigma(W_\mathrm{gate} a)} W_\mathrm{up} a
\end{equation*}
So if we use an $\ell_2$-loss for the output of the MLP module alone, the corresponding local Hessians are:
\begin{align*}
H_{W_\mathrm{up}} &= \E[a a^T \otimes (c_\mathrm{up} c_\mathrm{up}^T \odot W_\mathrm{down}^T W_\mathrm{down})] &\text{where} \; c_\mathrm{up} &= \sigma (W_\mathrm{gate} a) \\
H_{W_\mathrm{gate}} &= \E[a a^T \otimes (c_\mathrm{gate} c_\mathrm{gate}^T \odot W_\mathrm{down}^T W_\mathrm{down})] &\text{where} \; c_\mathrm{gate} &= \sigma' (W_\mathrm{gate} a) \odot W_\mathrm{up} a
\end{align*}

Note that the Schur product theorem guarantees that the Hadamard product of two positive (semi-)definite matrices is itself positive (semi-)definite.

\subsection{Choosing a Kronecker Approximation}
\label{app:hessianfactor}

In the previous subsection, we considered both a global and a local approach for computing Hessians for the linear unit $W \in \R^{m \times n}$.
In both cases the resulting Hessian has the form
\begin{equation*}
    H = \E[a a^T \otimes b b^T]
\end{equation*}
where $a, b$ may be vectors or matrices, and the expectation is taken over the calibration data.

In both the global and local settings, $a$ is the input to the linear unit. In the global setting, $b$ is the gradient of the global loss with respect to the output of the linear unit, and in the local MLP setting $b = D_{c_\mathrm{up}} W_\mathrm{down}^T$ or $b = D_{c_\mathrm{gate}} W_\mathrm{down}^T$.

However, this Hessian is an \emph{expectation} over Kronecker-factored matrices, and therefore is itself \emph{not} a Kronecker-factored matrix, in general.
We therefore seek a Kronecker product approximation of the form $H \approx A \otimes B$.

One simple approach to obtain such an approximation is to assume that $a,b$ are independent, so that the Kronecker product can be pulled out of the expectation. This gives
\begin{equation*}
    H \approx \E[a a^T] \otimes \E[b b^T].
\end{equation*}
We refer to this as the \emph{Marginal} approximation.

Another approach is based on a power iteration scheme that aims to minimize $\lVert H - A \otimes B \rVert_F$.
We refer to this as the \emph{Frobenius} approximation.
This technique has been employed by \citet{tseng2026modelpreserving} for YAQA, and we explain it in detail below.

This problem of finding the best Kronecker-factored approximation of a matrix in Frobenius norm was studied by \citet{van1993approximation}.
They define a rearrangement operator $\cR$ which maps an $mp \times nq$ matrix to an $mn \times pq$ matrix.
This operator $\cR$ is linear, invertible, and satisfies:
\begin{equation*}
    \lVert \cR(M) \rVert_F = \lVert M \rVert_F \qquad \langle \cR(M), \cR(N) \rangle = \langle M, N \rangle \qquad \cR(M \otimes N) = \vecop(M) \vecop(N)^T
\end{equation*}
Since $\cR$ preserves the Frobenius norm, minimizing $\lVert H - A \otimes B \rVert_F$ is equivalent to minimizing $\lVert \cR(H) - \vecop(A) \vecop(B)^T\rVert_F$.
Thus, the best Kronecker approximation of $H$ is obtained from the best rank-1 approximation of $\cR(H)$.
In our setting,
\begin{equation*}
    \cR(H) = \E[\cR(a a^T \otimes b b^T)] = \E[\vecop(a a^T) \vecop(b b^T)^T]
\end{equation*}
The best rank-1 approximation to a matrix is given by the outer product of the top left and top right singular vectors, scaled by the top singular value.
Since the algorithms in this paper are invariant under multiplying the Hessian by a scalar, we only need the top left and top right singular vectors of $\cR(H)$ up to scaling.
These singular vectors can be found using the power iteration.
Given a matrix $M$, and initial vectors $u_0$ and $v_0$, the iterations (without normalization) take the form:
\begin{equation*}
    u_{t+1} = M v_t \qquad v_{t+1} = M^T u_{t+1}
\end{equation*}
Substituting $M = \cR(H)$ and identifying $u = \vecop(A)$ and $v = \vecop(B)$ gives:
\begin{equation*}
    \vecop(A_{t+1}) = \E[\vecop(a a^T) \vecop(b b^T)^T \vecop(B_t)] = \E[\vecop(a a^T) \tr(b b^T B_t) ]
\end{equation*}
Applying $\vecop^{-1}$ and using the cyclic property of the trace yields:
\begin{equation*}
    A_{t+1} = \E[\tr(b^T B_t b) \cdot a a^T] \qquad B_{t+1} = \E[\tr(a^T A_{t+1} a) \cdot b b^T]
\end{equation*}
Each iteration requires a separate pass over the calibration data.
Compared to the Marginal approximation above, this approach weights the outer products $a a^T$ by $\tr(b^T B_t b)$.
If $a$ and $b$ are both vectors, as in our global Hessian case, the arguments of the trace operators above become scalars, and one can drop the ``$\tr$'' from the expressions.
For the local MLP Hessians, we have $b = D_{c_\mathrm{up}} W_\mathrm{down}^T$ or $b = D_{c_\mathrm{gate}} W_\mathrm{down}^T$, and we find (in the former case):
\begin{equation*}
    \tr(b^T B_t b) = \tr(W_\mathrm{down} D_{c_\mathrm{up}} B_t D_{c_\mathrm{up}} W_\mathrm{down}^T) = c_\mathrm{up}^T (B_t \odot W_\mathrm{down}^T W_\mathrm{down}) c_\mathrm{up}
\end{equation*}

A third approximation method called \emph{FlipFlop} was introduced by \citet{birnick2026waterkron}.
Algorithmically, it utilizes the same alternating iterations as the Frobenius approximation method, but it is ``reciprocal'' in the sense that, inside the expectation, $B_t$ is replaced by $B_t^{-1}$, and $A_{t+1}$ is replaced by $A_{t+1}^{-1}$.
While the Frobenius iteration minimizes $\lVert H - A \otimes B \rVert_F$, the FlipFlop iteration minimizes $\Phi(H, A \otimes B)$, which they call the Kronecker-Hessian mismatch factor:
\begin{equation*}
\Phi(H,A \otimes B) := \frac{\tr(H (A \otimes B)^{-1}) / mn}{|H|^{1/mn} \cdot |A \otimes B|^{-1/mn}}
\end{equation*}
It is specifically derived for the purpose of minimizing distortion of the two-sided GPTQ algorithm:
\citet{birnick2026waterkron} show theoretically and empirically that FlipFlop outperforms the other Kronecker approximation methods in this context.

\end{document}